\documentclass[conference,letterpaper]{IEEEtran}
\usepackage[utf8]{inputenc} 
\usepackage[T1]{fontenc}  
\usepackage{mathrsfs}
\usepackage[cmex10]{amsmath}
\usepackage{%
algorithm,%
algpseudocode,%
amsfonts,%
amssymb,%
amsthm,%
bbm,%
caption,%
cite,%
comment,%
extarrows,%
float,%
graphicx,%
indentfirst,%
listings,%
paralist,%
pgfplots,%
subcaption,%
tikz,%
todonotes,%
url,%
xcolor,%
xspace,%
makecell,%
booktabs,%
nicefrac,%
microtype,%
wrapfig
}

\usepackage{epstopdf}

\pgfplotsset{plot coordinates/math parser=false}
\usepgflibrary{%
patterns,%
}

\usetikzlibrary{%
arrows,%
automata,%
chains,%
decorations.pathreplacing,%
intersections,%
positioning,%
shapes,%
}

\IEEEoverridecommandlockouts  

\newtheorem{lemma}{Lemma}
\newtheorem{theorem}{Theorem}

\newtheorem{assumption}{Assumption}

\theoremstyle{definition}
\newtheorem{definition}{Definition}

\theoremstyle{remark}

\definecolor{darkgreen}{rgb}{0.0, 0.5, 0.0}

\newcommand{\set}[1]{\left\lbrace#1\right\rbrace}

\newcommand{\E}{\mathbb{E}}

\newcommand{\R}{\mathbb{R}}
\newcommand{\Z}{\mathbb{Z}}

\newcommand{\cD}{\mathcal{D}}

\newcommand{\cG}{\mathcal{G}}

\newcommand{\cM}{\mathcal{M}}

\newcommand{\cO}{\mathcal{O}}

\newcommand{\cV}{\mathcal{V}}

\renewcommand{\ge}{\geqslant}
\renewcommand{\le}{\leqslant}
\newcommand{\iid}{\emph{i.i.d.}\ }

\newcommand{\mix}{\mathrm{mix}}
\newcommand{\Prb}{\mathrm{Pr}}

\definecolor{bleudefrance}{rgb}{0.19, 0.55, 0.91}

\title{\huge RW-LoRA: Communication-Efficient Decentralized LoRA Fine-Tuning via Random Walks}
 
\author{Xingran Chen, Rohit Bhagat, Ghadir Ayache, Rawad Bitar, Yanmin Gong, and Salim El Rouayheb
\IEEEcompsocitemizethanks 	 
{
\IEEEcompsocthanksitem  Xingran Chen is with the Engineering Systems and Design Pillar, Singapore University of Technology and Design, Singapore 487372	(E-mail: xingranc@ieee.org).
\IEEEcompsocthanksitem Ghadir Ayache is with the LinkedIn, New York, NY 10118, USA (E-mail: ayache.ghad@gmail.com).
\IEEEcompsocthanksitem  Rawad Bitar is with the Chair of Communications Engineering, Technical University of Munich, 80333 Munich, Germany	(E-mail: 
rawad.bitar@tum.de).
\IEEEcompsocthanksitem  Yanmin Gong is with School of Engineering Medicine and the Department of Computer Science,  Texas A\&M University, College Station, TX 77843, USA	(E-mail: 
yanmin.gong@tamu.edu).
\IEEEcompsocthanksitem  Rohit Bhagat and Salim El Rouayheb are with Department of Electrical and Computer Engineering,  Rutgers University, Piscataway Township, NJ 08854, USA	(E-mail: \{rb1395, sye8\}@scarletmail.rutgers.edu).

\IEEEcompsocthanksitem This paper has been accepted by IEEE ITW 2026.
}
}

\begin{document}

\maketitle

\begin{abstract}
Parameter-efficient fine-tuning methods such as LoRA have become a standard approach for adapting large foundation models. Adopting fine-tuning to distributed settings faces several challenges. Most existing distributed LoRA methods rely on centralized aggregation, and gossip-based decentralized LoRA requires repeated synchronization among multiple model copies. Both methods incur significant communication overhead and introduce errors due to simultaneous aggregation of multiple model updates.  In this paper, we take a different perspective and propose a random-walk-based LoRA fine-tuning scheme. Instead of maintaining multiple model replicas, a single model token traverses the network and is updated sequentially using local fine-tuning objectives. This design eliminates the need for global synchronization, substantially reduces communication and computation costs, and avoids aggregation errors. We provide rigorous convergence guarantees for non-convex objectives under standard assumptions. Through empirical results on multiple NLP tasks and graph topologies, we show that the proposed method achieves competitive task performance with substantially less communication and computation than gossip-based LoRA. 
\end{abstract}

\section{Introduction}

Foundation models such as GPT-4 \cite{openai2024gpt4technicalreport}, LLaMA \cite{journals/corr/abs-2302-13971}, and BERT \cite{Devlin2019BERTPO} have transformed artificial intelligence, achieving strong performance on tasks ranging from translation to summarization \cite{ghiasvand-etal-2025-decentralized, Bommasani2021OnTO}. To achieve strong performance on domain-specific tasks, these models often need  to be adapted using local or task-specific data. However, their scale, often  ranging from $10^8$ to $10^{11}$ parameters, makes full fine-tuning  computationally expensive, communication-heavy, and prone to overfitting.  Parameter-efficient fine-tuning (PEFT) methods, especially low-rank adaptation  (LoRA), address this challenge by updating only a small number of additional  parameters while keeping the backbone model frozen \cite{hu2022lora}. As a  result, LoRA substantially reduces the computational and communication costs of  fine-tuning.

The benefits of LoRA are especially important in distributed and  resource-constrained applications, such as healthcare, edge personalization, and  enterprise AI, where privacy, bandwidth, and data-sovereignty constraints often prevent raw data from being centralized. In such settings, fine-tuning must be  performed through distributed collaboration across multiple data owners or edge  nodes. Among them, federated LoRA methods  \cite{10.5555/3737916.3738624,bian2025fedalt,Yi2023FedLoRAMP,yan2025federated, guo2025selective,singhal2025fedexlora,FedMomentum} have become a predominant  approach. However, these methods rely on centralized aggregation through a parameter server, which introduces communication and memory overhead and creates  a single point of failure. To remove the need for central coordination, Ghiasvand et~al.\ \cite{ghiasvand-etal-2025-decentralized} recently proposed a gossip-based  decentralized LoRA method. Nevertheless, gossip-based methods require nodes to  exchange model updates with their neighbors repeatedly, which can still incur  substantial communication overhead, especially on dense or bandwidth-limited  networks.

More importantly, LoRA updates are of the form $W=BA$ where the model $W\in\mathbb{R}^{d_1\times d_2}$ is factored into two low-rank matrices $A \in \mathbb{R}^{r \times d_2}$ and $B \in \mathbb{R}^{d_1 \times r}$ with $r \ll \min\{d_1,d_2\}$. In federated and gossip-based learning, the local models $W_i$ of the nodes need to be averaged at every round. To avoid communicating their local $W_i= B_iA_i$'s, the nodes send their factor matrices $A_i$ and $B_i$. Aggregation can happen in two ways, each with its own drawback: (i) averaging the factors introduces a \emph{bilinear mismatch} \cite{wang2024flora,singhal2024fedex,zhang2026fedrotlora,xia2026sefora}, since $\sum_i B_iA_i \neq (\sum_i B_i) (\sum_i A_i)$; and (ii) computing $W_i = B_iA_i$, averaging them and doing SVD to get a low-rank factorization of $W$; introducing additional computation and potential truncation errors.

To address these limitations, we advocate token-based random-walk learning as a scalable alternative. A token, consisting of the current pair of low-rank factors $(A_t,B_t)$, traverses the graph by moving at each iteration from its current node to a randomly selected neighbor, where it is updated using that node's local data \cite{10.5555/3618408.3618786,doi:10.1137/08073038X,10.5555/3327546.3327656}. Random-walk learning differs  from consensus- and gossip-based methods \cite{10.5555/3737916.3738624, ghiasvand-etal-2025-decentralized}: rather than maintaining and synchronizing local models kept at each node, it propagates a single   model through the network. This offers two main advantages. First, at each iteration, communication and computation are devoted to advancing a single model trajectory rather than updating and reconciling local models. Second, the sequential update of the model at each visited node avoids the drawbacks of simultaneously aggregating multiple updates. While their inherently sequential nature may be less advantageous in applications requiring rapid convergence, random-walk methods offer an attractive alternative suitable for applications running in the background or requiring convergence with low communication and computation overhead.

Our contributions are summarized as follows:
\begin{compactenum}[(i)]
\item We propose random-walk-based LoRA fine-tuning as a low communication overhead alternative to gossip-based LoRA~\cite{ghiasvand-etal-2025-decentralized}. This design enables decentralized fine-tuning without a parameter server or synchronous neighbor-wise aggregation. We also provide rigorous convergence guarantees for non-convex objectives (Theorem~\ref{thm:Convergence}).

\item We conduct experiments on several GLUE benchmark tasks using complete and ring communication graphs. The training data are partitioned across nodes in an \iid manner. Compared with the gossip-based decentralized LoRA baseline~\cite{ghiasvand-etal-2025-decentralized}, RW-LoRA maintains comparable accuracies while significantly reducing both communication and computational overhead (Fig.~\ref{fig:overhead}). 

\end{compactenum}

\section{System Model}\label{sec:SystemModel}

We consider a set of nodes collaborating to fine-tune a large model. Nodes hold local iid datasets\footnote{Extending RW-LoRA to non-\iid data is left for future work.} that can be used for fine-tuning, and they are able to communicate only with their neighbors. We represent the communication network as a graph, where each vertex corresponds to a node and each edge indicates that the two corresponding nodes can communicate directly. A token carrying the model parameters is passed from node to node and moves on this graph.

\begin{definition}[Communication topologies and random walks]\label{defn:GenGph}
A communication topology is defined as a finite undirected graph $\cG \triangleq (\cV, E)$ with the set of nodes $\cV = [N]$ and the set of edges $E \subseteq \binom{\cV}{2}$. 
A random walk $\set{v_t}_{t\ge0}:\Omega\to\cV^{\Z_+}$ on this graph can be defined by the common transition probability matrix $P:\cV\to \cM(\cV)$, 
where the probability of transition from node $u$ to node $v$ in one time step at time $t\in\Z_+$, is 
$P_{uv} \triangleq \Prb\left(v_{t+1}= v \,\middle|\, v_t=u\right)$.
\end{definition}

Without loss of generality, we assume that $P_{uv} > 0$ for all nodes $v$ connected to node $u$ in graph $\cG$.  

The stationary distribution describes the long-run fraction of time that the random walk spends at each node. Let $P$ be the transition matrix defined in Definition~\ref{defn:GenGph}. A probability $\pi$ is called a stationary distribution if $\pi = \pi P$.

Let $\pi_0$ denote the initial distribution, and let $\pi_t=P^t\pi_0$ denote the distribution at time step $t$. For $\epsilon>0$, the mixing time $\tau_{\mix}(\epsilon)$ of $\set{v_t}_{t\ge0}$ is defined as \cite[Definition~2]{10.5555/3618408.3618786}:
\begin{align*}
\tau_{\mix}(\epsilon) = \inf\set{t\ge1|\forall \pi_0, d_{\text{TV}}(P^t\pi_0, \pi)\le\epsilon},
\end{align*}
where $d_{\text{TV}}(\cdot,\cdot)$ is the total-variance distance and $\pi_0$ is the initial distribution.

We consider the decentralized stochastic optimization
\begin{align}\label{eq:OriginalF}
\min_{\bf w}\,f({\bf w}) = \min_{\bf w}\,\E_{v\sim\pi}\left[f_v({\bf w})\right],
\end{align}
where ${\bf w}$ denotes the set of parameters to be fine-tuned (details are provided in Section~\ref{subsec:Low-Rank Decomposition} and Section~\ref{sec:Random Walk-based LoRA Scheme}) and $\pi$ is a probability distribution over the node set $\cV$. In our setting, $f$ is the global fine-tuning objective, while $f_v$ is the local fine-tuning objective associated with the data stored at node $v$. Let $\cD_v$ denote the local data distribution specific to user $v$, then the local fine-tuning objective $f_v$ is given in stochastic form:
\begin{align}
f_v({\bf w}) = \E_{\xi_v\sim\cD_v}\left[F_v({\bf w}; \xi_v)\right].
\end{align}

Let $\{v_t\}_{t\ge0}$ be a random walk on $\cV$ with stationary distribution $\pi$. At time $t$, the model token is located at node $v_t$ and is updated using only the local objective at that node \cite{10.5555/3618408.3618786, doi:10.1137/08073038X, 10.5555/3327546.3327656}:
\begin{align}\label{eq:RWupdatingrule}
{\bf w}_{t+1}={\bf w}_t - \eta \nabla f_{v_t}({\bf w}_t),
\end{align}
where $\eta>0$ is the stepsize. \eqref{eq:RWupdatingrule} provides only a general random-walk update rule for random-walk-based learning. Its adaptation to the LoRA setting is presented in Section~\ref{sec:Random Walk-based LoRA Scheme}.



\section{Preliminaries on LoRA}\label{sec:Preliminaries}
\subsection{Low-Rank Adaptation}\label{subsec:Low-Rank Decomposition}

Consider a pre-trained model $\Phi_0$ consisting of multiple layers parameterized by weight matrices. Let $W_0 \in \mathbb{R}^{d_1 \times d_2}$ denote the weight matrix of one layer. In full fine-tuning, one would learn an update matrix $\Delta W \in \mathbb{R}^{d_1 \times d_2}$ and use $W = W_0 + \Delta W$.
To increase the efficiency of fine-tuning, LoRA keeps $W_0$ fixed and restricts the update to a low-rank form \cite{hu2022lora} $\Delta W = BA$, where $A \in \mathbb{R}^{r \times d_2}$, $B \in \mathbb{R}^{d_1 \times r}$, and $r \ll \min\{d_1,d_2\}$. Thus, the adapted weight matrix becomes
\begin{align}\label{eq:LoRA0}
W = W_0 + BA .
\end{align}

During the fine-tuning phase, LoRA optimizes the factors (also called adapters) $A$ and $B$ instead of directly updating $\Delta W$. This reduces the number of trainable parameters from $d_1d_2$ to $r(d_1+d_2)$ and hence reduces the required computation and memory resources \cite{10.5555/3737916.3738624}. Throughout the paper, we refer to $r$ as the LoRA rank, which is typically chosen from $\{2,4,8,16\}$. 

\subsection{Federated LoRA-based Fine-Tuning}\label{subsec:FedDecFine-Tuning}
LoRA has become a popular approach for adapting foundation models \cite{hu2022lora}. To enable efficient fine-tuning across distributed data sources, LoRA has recently been incorporated into federated learning. Representative methods include FLoRA \cite{10.5555/3737916.3738624}, FedALT \cite{bian2025fedalt}, FedLoRA \cite{Yi2023FedLoRAMP}, FRLoRA \cite{yan2025federated}, FedSA-LoRA \cite{guo2025selective}, FedEx-LoRA \cite{singhal2025fedexlora}, and FedMomentum \cite{FedMomentum}. These methods mainly differ in how they aggregate or personalize LoRA adapters across clients. For example, FLoRA, FedALT, and FedLoRA address heterogeneous or personalized LoRA adapters through stacking, personalized components, or global-local knowledge exchange. Other methods, including FRLoRA, FedSA-LoRA, FedEx-LoRA, and FedMomentum, improve aggregation efficiency or stability through residual updates, selective factor sharing, exactness-preserving corrections, or momentum-based aggregation.  Most existing federated LoRA methods still rely on centralized aggregation through a server. As a result, they inherit the communication and coordination bottlenecks of classical federated learning, which become increasingly limiting for large foundation models.

\subsection{Token-based Random-Walk Learning}\label{subsec:TokenRWLearning}

Decentralized learning bypasses the need of a central node coordinating the process. The closest to our work is the recent work of \cite{ghiasvand-etal-2025-decentralized}. The authors studied fully decentralized, gossip-based LoRA fine-tuning algorithm. They established convergence guarantees for non-convex objectives under standard assumptions. In contrast, our work studies a random-walk-based LoRA method that avoids repeated server aggregation and neighbor-wise synchronization. This design reduces both communication and computational overhead.

The literature on random-walk-based learning is broad, and we provide only a brief overview here. Seminal work in this area includes \cite{Johansson2007ASP, Lopes2007IncrementalAS, doi:10.1137/08073038X}, which laid the foundation for optimization methods driven by Markov chain sampling. There are two
directions to design and analyze token algorithms. First, \cite{Johansson2007ASP} studied incremental optimization and subdifferentials methods under Markov chain sampling. Second, \cite{doi:10.1137/08073038X, Lopes2007IncrementalAS} developed improved convergence guarantees for stochastic gradient and mirror-descent methods under Markovian sampling. A common feature of these results is that the convergence rate depends on the mixing time $\tau_{\mathrm{mix}}$ of the underlying Markov chain, often through terms such as
$\cO\left(\frac{\tau_{\mathrm{mix}}}{T}\right)+\cO\left(\sqrt{\frac{\tau_{\mathrm{mix}}}{T}}\right)$. For a more complete account of random-walk-based learning and token algorithms, we refer the reader to \cite{10.5555/3618408.3618786}.
Recently, random-walk-based learning has also been studied for data privacy~\cite{ayache2021private,egger2025source} and robustness to probabilistic and malicious node failures~\cite{egger2024self,CIL,khalesi2026fundamental}.

\section{The RW-LoRA Algorithm}\label{sec:Random Walk-based LoRA Scheme}

In this section, we present the RW-LoRA algorithm, summarized in Algorithm~\ref{alg:rw-lora}.

\begin{algorithm}[htbp]
\caption{RW-LoRA Algorithm}\label{alg:rw-lora}
\begin{algorithmic}[1]
\Require Graph $\mathcal{G}=(\mathcal{V},E)$, transition matrix $P$, stepsize $\eta$, LoRA rank $r$, number of iterations $T$, pre-trained weight matrix $W_0 \in \mathbb{R}^{d_1\times d_2}$
\State Initialize $[A^{(0)}]_{ij} \sim \mathcal{N}(0,\sigma^2)$ and $B^{(0)}={\bf 0}$
\State Initialize the random walk at node $v_0 \in \mathcal{V}$
\For{$t=0,1,\ldots,T-1$, at node $v_t$}
    \State Compute $W^{(t)} = W_0 + B^{(t)}A^{(t)}$.
    \State 
    Update the LoRA factors via \eqref{eq:recursionAlocal} -- \eqref{eq:recursionB}.
    \State Sample the next node according to $v_{t+1} \sim P_{v_t,\cdot}$.
    \State Transfer the token $(A^{(t+1)},B^{(t+1)})$ 
    to node $v_{t+1}$.
\EndFor
\State \Return $A^{(T)}, B^{(T)}$ and $W^{(T)}=W_0+B^{(T)}A^{(T)}$
\end{algorithmic}
\end{algorithm}

As discussed in Section~\ref{subsec:Low-Rank Decomposition}, we fine-tune the model by optimizing the low-rank factors $A$ and $B$. Specifically, the optimization problem in \eqref{eq:OriginalF} can be written as:
\begin{align}\label{eq:LoRA1}
\min_{A, B} f(W) = \min_{A, B} \E_{v\sim\pi}\left[f_v(W)\right],
\end{align}
where $A \in \mathbb{R}^{r \times d_2}$, $B \in \mathbb{R}^{d_1 \times r}$, $f_v:\R^{d_1\times d_2}\to\R$ is the local fine-tuning objective of node $v$, and $\pi$ is a pre-determined target distribution.

For simplicity, we use $\nabla$ to denote the gradient with respect to $W$, i.e., $\nabla_W$, and use $\nabla_A$ and $\nabla_B$ to denote the gradients with respect to $A$ and $B$, respectively. We define 
\begin{align}\label{eq:nablaW}
\nabla \tilde{f}_v(W) \triangleq \nabla F_v(W;\xi_v).
\end{align}

Given a target sampling distribution $\pi$ \big(assuming $\pi_v>0$ for all $v\in\cV$\big), the Metropolis-Hastings algorithm \cite{doi:10.1137/08073038X} provides a principled way to construct a RW  on $\cG$ with a transition matrix $P$ such that $\pi$ is the stationary distribution of the Markov chain defined by $P$.\footnote{The transition matrix $P$ depends on both the target sampling distribution $\pi$ and the graph $\cG$. We re-parameterize it as $P_{\pi,\cG}\in\R_+^{N\times N}$, and simply write $P$ when there is no risk of confusion.} Let $\{v_t\}_{t\ge 0}$ be a random walk on $\cV$ that evolves according to the transition matrix $P$. As mentioned in Section~\ref{sec:SystemModel}, \eqref{eq:RWupdatingrule} provides a general random-walk update rule for vector-valued parameters. Since LoRA fine-tuning optimizes matrix-valued parameters, we extend this update rule to the two LoRA matrices $A$ and $B$: At time $t$, the model token is located at node $v_t$, where the local fine-tuning objective is used to update the LoRA factors for $K$ local steps \cite{ghiasvand-etal-2025-decentralized, hu2022lora}. Specifically, for $k\in\set{0,1,\cdots,K-1}$,
\begin{align}
&A^{(t, k+1)} = A^{(t, k)} - \eta \nabla_A \tilde{f}_{v_t}(W^{(t, k)})\label{eq:recursionAlocal},\\
&B^{(t, k+1)} = B^{(t, k)} - \eta \nabla_B \tilde{f}_{v_t}(W^{(t, k)})\label{eq:recursionBlocal},
\end{align}
After the $K$ local steps, the updated LoRA matrices are given by
\begin{align}
&A^{(t+1)} = A^{(t, K)}\label{eq:recursionA},\\
&B^{(t+1)} = B^{(t, K)}\label{eq:recursionB}.
\end{align}
The local updates are initialized with $A^{(t, 0)} = A^{(t)}$ and $B^{(t, 0)} = B^{(t)}$. The initial LoRA matrices are initialized as $A^{(0)} \in \mathbb{R}^{r \times d_2}$, $[A^{(0)}]_{ij} \overset{\iid}{\sim} \mathcal{N}(0,\sigma^2)$, $B^{(0)}={\bf 0}\in\R^{d_1\times r}$, and $W^{(0)}=W_0+B^{(0)}A^{(0)}$, as adopted in \cite{ghiasvand-etal-2025-decentralized}.

\section{Theoretical Results}\label{sec:TheoreticalResults}

This section establishes the convergence of RW-LoRA in the non-convex setting under Assumptions~\ref{assu:Submultiplicative}--\ref{assu:BoundedPara}. These assumptions are stated and discussed immediately after the theorem.
\begin{theorem}\label{thm:Convergence}
Let Assumptions~\ref{assu:Submultiplicative} $\sim$ \ref{assu:BoundedPara} hold. Let $\epsilon<\frac{1}{2}$, $\pi_{\min}=\min_{v\in\cV}\set{\pi_v}$, $\tau\ge\tau_{\mix}(\pi_{\min}\epsilon)$, $T\ge \tau^2$, $\eta\le\frac{1}{4L\sqrt{T}}$, we obtain\footnote{Here, $\widetilde{\cO}(\cdot)$ hides logarithmic factors in the relevant problem parameters. For example,  $\widetilde{\cO}(1/\sqrt{T})$ may represent terms such as  $\cO(\log T/\sqrt{T})$ or $\cO(\log(cT)/\sqrt{T})$, depending on the context.}:
\begin{align}\label{eq:convergence}
\frac{1}{T}&\sum_{t=1}^{T}\left(\E\left[\|\nabla_{A} f(W^{(t)})\|^2\right]+\E\left[\|\nabla_{B} f(W^{(t)})\|^2\right] \right)\nonumber\\
\le& \widetilde{\cO}\left(\epsilon^2 + \frac{L\tau}{\sqrt{T}}+\frac{L\tau}{T}\right).
\end{align}
\end{theorem}

First, Theorem~\ref{thm:Convergence} shows that the proposed RW-LoRA algorithm converges to a stationary point in terms of the following first-order stationarity measure \cite{ghiasvand-etal-2025-decentralized}:
\begin{align*}
\E\left[\|\nabla_A f(W^{(t)})\|\right] + \E\left[\|\nabla_B f(W^{(t)})\|\right].
\end{align*}
This criterion is natural for LoRA fine-tuning because the trainable variables are the low-rank factors $A$ and $B$, rather than the full update matrix $\Delta W$. Hence, stationarity should be evaluated with respect to these actual optimization variables. The complete proofs are provided in Appendix~\ref{App:Convergence}.

Second, Theorem~\ref{thm:Convergence} indicates that the convergence rate depends on the mixing time $\tau$ of the underlying Markov chain. In particular, our bound has the form $$\widetilde{\cO}\left(\epsilon^2 + \frac{L\tau}{\sqrt{T}} + \frac{L\tau}{T}\right).$$ This is broadly consistent with the rates obtained in standard random-walk learning for vector-valued objectives, where the dependence on the number of iterations typically appears as $\cO\left(\frac{\tau_{\mathrm{mix}}}{T}\right) + \cO\left(\sqrt{\frac{\tau_{\mathrm{mix}}}{T}}\right)$ \cite{Johansson2007ASP, Lopes2007IncrementalAS, doi:10.1137/08073038X}. There are, however, two differences. First, our result is stated in $\widetilde{\cO}$ notation, whereas the standard vector-valued setting is often expressed in $\cO$ notation. Second, the coefficient of the
$1/\sqrt{T}$ term scales linearly with the mixing time $\tau$ in our LoRA setting, while it scales as $\sqrt{\tau_{\mathrm{mix}}}$ in the standard vector-valued case. This difference arises from the low-rank matrix factorization in LoRA: the trainable variables are the matrices $A$ and $B$, and the update $\Delta W = BA$ introduces additional coupling between the two factors.

\begin{assumption}[Submultiplicative property]\label{assu:Submultiplicative}
The matrix norm $\|\cdot\|$ used in this paper is submultiplicative, i.e., $\|BA\| \le \|B\|\|A\|$ for any compatible matrices $A$ and $B$. 
\end{assumption}
This assumption is mild and is satisfied by most commonly used matrix norms, such as the spectral norm, Frobenius norm, $1$-norm, and $\infty$-norm.

\begin{assumption}[$L$-smooth, Assumption~4.1 in \cite{ghiasvand-etal-2025-decentralized}]\label{assu:Lsmooth}
We assume that each local fine-tuning objective $f_v$ is $L$-smooth, i.e., for all $W, W'\in\R^{d_1\times d_2}$, we have
$$\|\nabla f_v(W) - \nabla f_v(W')\|\le L\|W-W'\|.$$
\end{assumption}

\begin{assumption}[Bounded gradient, Assumption~4.2 in \cite{ghiasvand-etal-2025-decentralized}]\label{assu:Bounds0}
We assume that the stochastic gradients are unbiased and that their expected squared norm remains uniformly bounded:
\begin{align*}
&\E_{\xi_v\sim\cD_v}\left[\nabla \tilde{f}_v(W)\right]=\nabla f_v(W),\nonumber\\
&\E_{\xi_v\sim\cD_v}\left[\|\nabla \tilde{f}_v(W)\|^2\right]\le c^2,\,\, v\in\cV,
\end{align*}
where $\xi_v$ represents a randomly sampled subset of training data from $v$-th node.
\end{assumption}

Assumption~\ref{assu:Bounds0} is standard in decentralized optimization \cite{ghiasvand-etal-2025-decentralized, 10.5555/3737916.3738624, Yi2023FedLoRAMP, guo2025selective}. It helps bound the magnitude of the update steps, as in the analysis in Lemmas~$7$ and $11$ of \cite{ghiasvand-etal-2025-decentralized}. For notational simplicity, we omit the subscript $\xi_v\sim\cD_v$ when it is clear from the context.

\begin{assumption}[Bounded variance, Assumption~5.1 in \cite{Even}]\label{assu:BoundedVar}
There exists $\sigma^2$ such that for all $v\in\cV$ and all $W\in\R^{d_1\times d_2}$, we have:
$$\E\left[\|\nabla\tilde{f}_v(W)-\nabla f(W)\|^2\right]\le\sigma^2.$$
\end{assumption}
In Assumption~\ref{assu:BoundedVar}, the expectation $\E$ is taken over two sources of randomness:  (i) the randomness in data sampling at node $v$, i.e., $\xi_v\sim\cD_v$, and (ii) the randomness in the global objective, as defined in \eqref{eq:LoRA1}. It is used to control the stochastic noise arising from the random-walk sampling process \cite[Lemmas C.1 \& C.2]{Even}.

Assumptions~\ref{assu:Bounds0} and~\ref{assu:BoundedVar} are both necessary and play distinct roles in our analysis. For vector-valued optimization variables, one of these assumptions is often sufficient, since the other is not needed for the convergence proof. In our setting, however, the optimization variables are the matrix-valued LoRA factors $A$ and $B$, with $\Delta W=BA$. Because $A$ and $B$ are updated separately, dropping either assumption would prevent us from completing the convergence proof.

\begin{assumption}[Assumption~4.3 in \cite{ghiasvand-etal-2025-decentralized}]\label{assu:BoundedPara}
There exists constants $a>0$ and $b>0$ such that: $\|A^{(t)}\|\le a$, $\|B^{(t)}\|\le b$ for all $0\le t\le T$.
\end{assumption}

In Assumption~\ref{assu:BoundedPara}, the constants $a$ and $b$ are chosen uniformly over $t$. These constants may depend on the matrix dimensions, namely $r$, $d_1$, and $d_2$.

\section{Experiments}\label{sec:Experiments}

In this section, we compare the proposed RW-LoRA algorithm with a gossip-based decentralized LoRA baseline~\cite{ghiasvand-etal-2025-decentralized}. The results show that RW-LoRA significantly reduces both communication and computational overhead while avoiding the bilinear mismatch issue. 
\subsection{Experiment Settings}\label{subsec:ExperimentSettings}
\subsubsection{Setup}
We use RoBERTa-base with 125M parameters \cite{liu2019roberta} as the backbone model for all tasks. Following \cite{ghiasvand-etal-2025-decentralized}, we use a default LoRA rank of $r=16$, and additionally conduct an ablation study with $r \in\set{4,8,16,32}$. We use the default \textsc{AdamW} optimizer with learning rate $10^{-3}$. The random-walk fine-tuning process is run over $30$ nodes. At each communication round, the active node performs $K=10$ local training steps with batch size $b=32$ before passing the model token to the next node. For reproducibility, all five runs use a fixed random seed of $42$, except for the random-walk process, which uses seeds $1$, $2$, $3$, $4$, and $5$ across the five runs to generate different trajectories.

\subsubsection{Datasets}
We evaluate the proposed RW-LoRA algorithm on five GLUE datasets, covering two types of tasks:
\begin{compactenum}[(i)]
\item Sentence-pair classification: We fine-tune the base model on MRPC \cite{dolan2005automatically}, QQP \cite{Wang2018GLUEAM}, QNLI \cite{Wang2018GLUEAM, rajpurkar-etal-2016-squad}, and MNLI \cite{Wang2018GLUEAM, Williams2017ABC}. These tasks require the model to determine semantic equivalence, textual entailment, or whether a sentence contains the answer to a given question.
\item Sentiment classification: We fine-tune the base model on SST-2 \cite{socher-etal-2013-recursive}, where the task is to classify each sentence as positive or negative.
\end{compactenum}

\subsubsection{Baseline}
We compare RW-LoRA to the gossip-based decentralized LoRA method introduced in~\cite{ghiasvand-etal-2025-decentralized}.

\subsubsection{Topologies} In our decentralized framework, nodes communicate exclusively along the edges of a fixed communication graph that connects $30$ nodes. We focus on two topologies: ring graphs and complete graphs. These two topologies represent two extremes in network connectivity \cite{PacManAC, CIL}. The target distribution $\pi$ is chosen as the uniform distribution, and the corresponding transition probability matrix is constructed using the Metropolis--Hastings rule~\cite{doi:10.1137/08073038X}.

\begin{table*}[t]
\centering
\vspace*{0.02in}
\setlength{\abovecaptionskip}{2pt}
    \begin{tabular}{c|c|c|c|c|c}
        & MRPC & SST2 & QNLI & MNLI & QQP \\ \hline
        RW (Final Scores) 
        & \textbf{91.32} $\pm$ 0.87 
        & 93.28 $\pm$ 0.54 
        & \textbf{91.35} $\pm$ 0.22 
        & \textbf{84.50} $\pm$ 0.09 
        & \textbf{87.61} $\pm$ 0.40 \\ \hline
        Gossip (Final Scores) 
        & 89.26 & \textbf{93.58} & 90.75 & 83.59 & 86.59 \\ \hline
        RW (Best Scores) 
        & \textbf{91.56} $\pm$ 0.78 
        & 94.03 $\pm$ 0.64 
        & \textbf{91.98} $\pm$ 0.33 
        & \textbf{85.90} $\pm$ 0.63 
        & \textbf{88.77} $\pm$ 0.41 \\ \hline
        Gossip (Best Scores) 
        & 90.30 & \textbf{94.38} & 90.80 & 85.20 & 88.05
    \end{tabular}

    \caption{Accuracy under Complete Graph}
    \label{Table:Complete}

    \begin{tabular}{c|c|c|c|c|c}
        & MRPC & SST2 & QNLI & MNLI & QQP \\ \hline
        RW (Final Scores) 
        & \textbf{90.30} $\pm$ 0.87 
        & 92.96 $\pm$ 0.87 
        & \textbf{90.87} $\pm$ 0.48 
        & \textbf{84.83} $\pm$ 0.40 
        & 82.52 $\pm$ 9.68 \\ \hline
        Gossip (Final Scores) 
        & 83.26 & \textbf{93.12} & 90.33 & 83.72 & \textbf{86.63} \\ \hline
        RW (Best Scores) 
        & \textbf{91.61} $\pm$ 0.51 
        & 93.83 $\pm$ 0.34 
        & \textbf{91.51} $\pm$ 0.17 
        & \textbf{85.65} $\pm$ 1.05 
        & \textbf{87.87} $\pm$ 0.45 \\ \hline
        Gossip (Best Scores) 
        & 88.42 & \textbf{94.00} & 91.08 & 84.82 & 87.58
    \end{tabular}

    \caption{Accuracy under Ring Graph}
    \label{Table:Ring}
\end{table*}

\subsection{Experimental Results}\label{subsec:ExperimentResults}

\paragraph{Task Performance} 
The task performance of random-walk-based LoRA is reported in Tables~\ref{Table:Complete} \& \ref{Table:Ring}.  Table~\ref{Table:Complete} reports results on the complete graph, while Table~\ref{Table:Ring} reports results on the ring topology\footnote{In both tables, the performance metric for $\mathrm{MRPC}$ is F1, whereas the metric for all other datasets is accuracy. For simplicity of presentation, we use the term ``Accuracy'' in the tables to refer to the corresponding task performance metric.}. In both tables, \textit{Final Score} denotes the model performance at the last training round, whereas \textit{Best Score} denotes the highest performance achieved over all training rounds. Both methods are run until convergence, resulting in different convergence horizons: approximately $2500$ time steps for RW-LoRA and $180$ time steps for gossip-based LoRA. Although gossip-based LoRA converges in fewer rounds, each round involves neighbor-to-neighbor synchronization across multiple nodes, leading to approximately $54000$ local updates in total. By contrast, RW-LoRA activates only one node at each time step and requires approximately $25000$ local updates before convergence.

As shown in both tables, RW-LoRA performs comparably to gossip-based LoRA in both final and best accuracies on the $\mathrm{MRPC}$, $\mathrm{QNLI}$, $\mathrm{MNLI}$, and $\mathrm{QQP}$ tasks. On $\mathrm{SST2}$, RW-LoRA achieves final and best accuracies comparable to those of gossip-based LoRA. Overall, these results demonstrate that RW-LoRA consistently matches or exceeds the performance of gossip-based LoRA across all evaluated tasks. More importantly, RW-LoRA achieves comparable performance with substantially lower communication and computation overhead.

\paragraph{Communication and Computation Efficiency} 

Under gossip-based LoRA, each node exchanges updates with its neighbors in every communication round. Thus, the per-node communication cost is $O(\mathrm{degree}\times \mathrm{model\_size})$, here $0.59M$ for the ring graph and $8.555M$ for the complete graph, and the total network-wide cost per round is $O(\mathrm{network\_size}\times \mathrm{degree}\times \mathrm{model\_size})$, here $17.7M$ and $256.65M$ for the ring and complete graph, respectively. In contrast, random-walk-based LoRA involves only one active transmission per round, from the current node to the next node on the walk. Its total communication cost per round is therefore
$O(\mathrm{model\_size})$, here $0.295M$ for both graphs. For a fixed network size $N$ and a fixed number of local updates $K$, the computational cost scales linearly with the communication horizon. Specifically, it is $O(K\times \mathrm{network\_size}\times \mathrm{time\_horizon})$ for gossip-based learning and $K\times \mathrm{time\_horizon}$ for random-walk learning.

Fig.~\ref{fig:overhead} shows the results on the $\mathrm{QNLI}$ dataset. The dataset is partitioned across nodes in an \iid manner. Although we evaluate all methods on five datasets, the accuracy curves exhibit similar trends across datasets; therefore, we report the results for $\mathrm{QNLI}$ only. In Fig.~\ref{fig:overhead}, the $x$-axis denotes the communication overhead, measured by the total number of edge activations\footnote{Here, one edge activation corresponds to one communication event between two neighboring nodes.}, while the $y$-axis represents the accuracy or loss. As shown in Figs.~\ref{fig:communicationring} and \ref{fig:computationring}, RW-LoRA achieves comparable or higher accuracy and converges with significantly lower communication costs than gossip-based LoRA while requiring substantially lower communication costs.

\begin{figure}[thbp]
\centering
\begin{subfigure}{0.241\textwidth}
\centering
\includegraphics[width=\linewidth, height=0.75\linewidth]{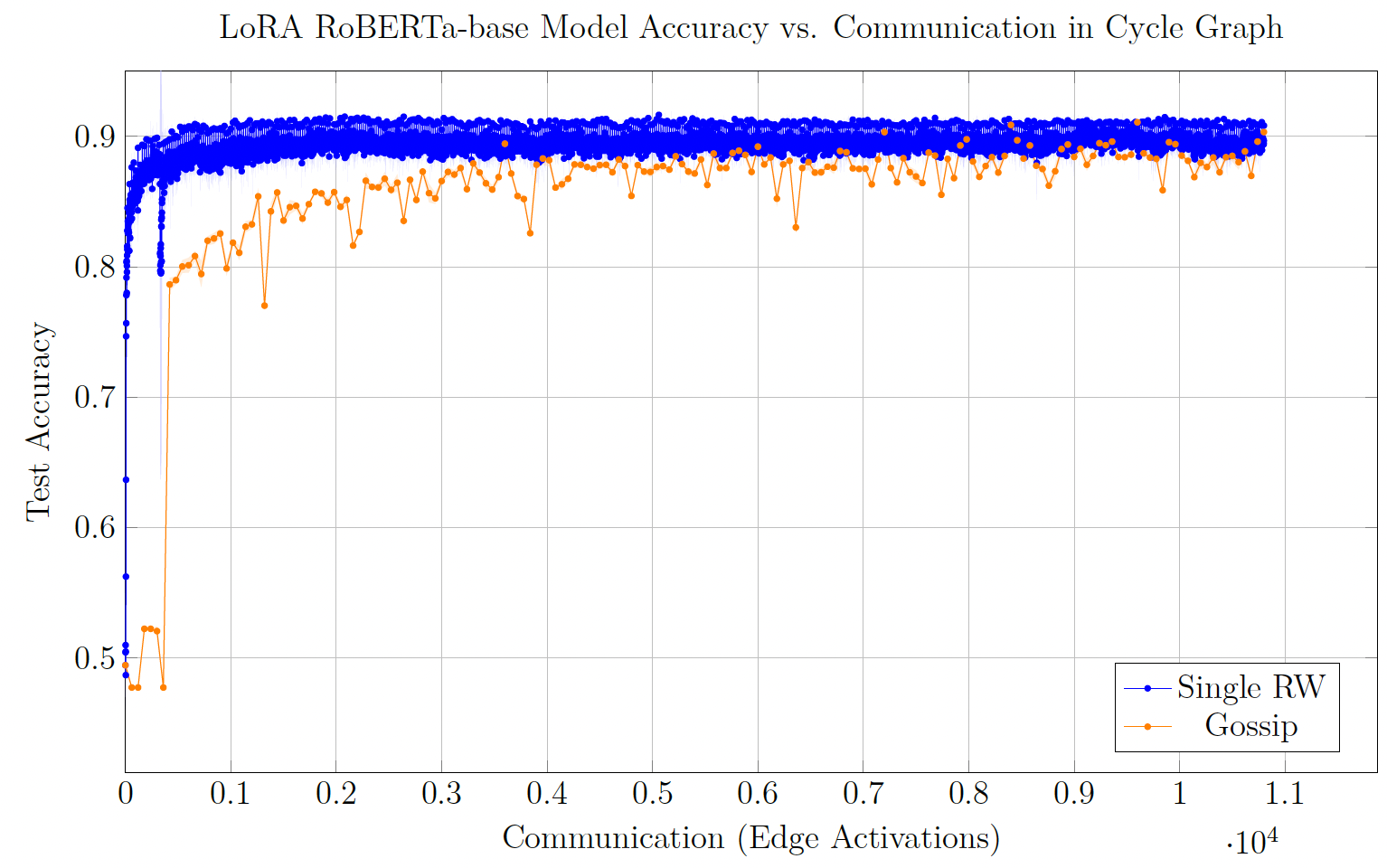}
\caption{}
\label{fig:communicationring}
\end{subfigure}
\begin{subfigure}{0.241\textwidth}
\centering
\includegraphics[width=\linewidth, height=0.75\linewidth]{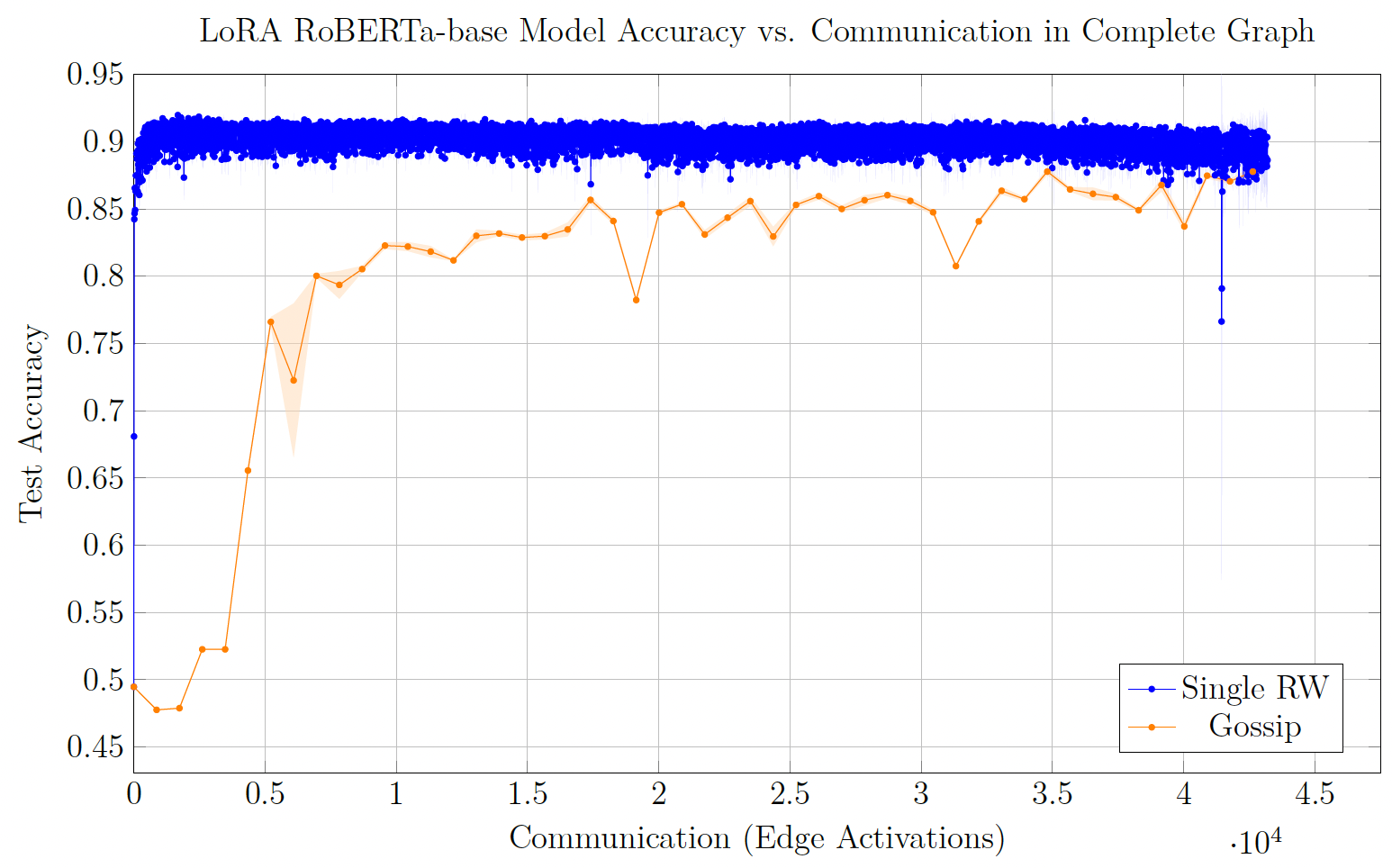}
\caption{}
\label{fig:computationring}
\end{subfigure}
\caption{Accuracy vs. communication overhead in a ring and complete graph.} 
\label{fig:overhead}
\end{figure}

\paragraph{Ablation Study}
\begin{wrapfigure}{r}{0.48\columnwidth}
\centering
\includegraphics[width=0.9\linewidth]{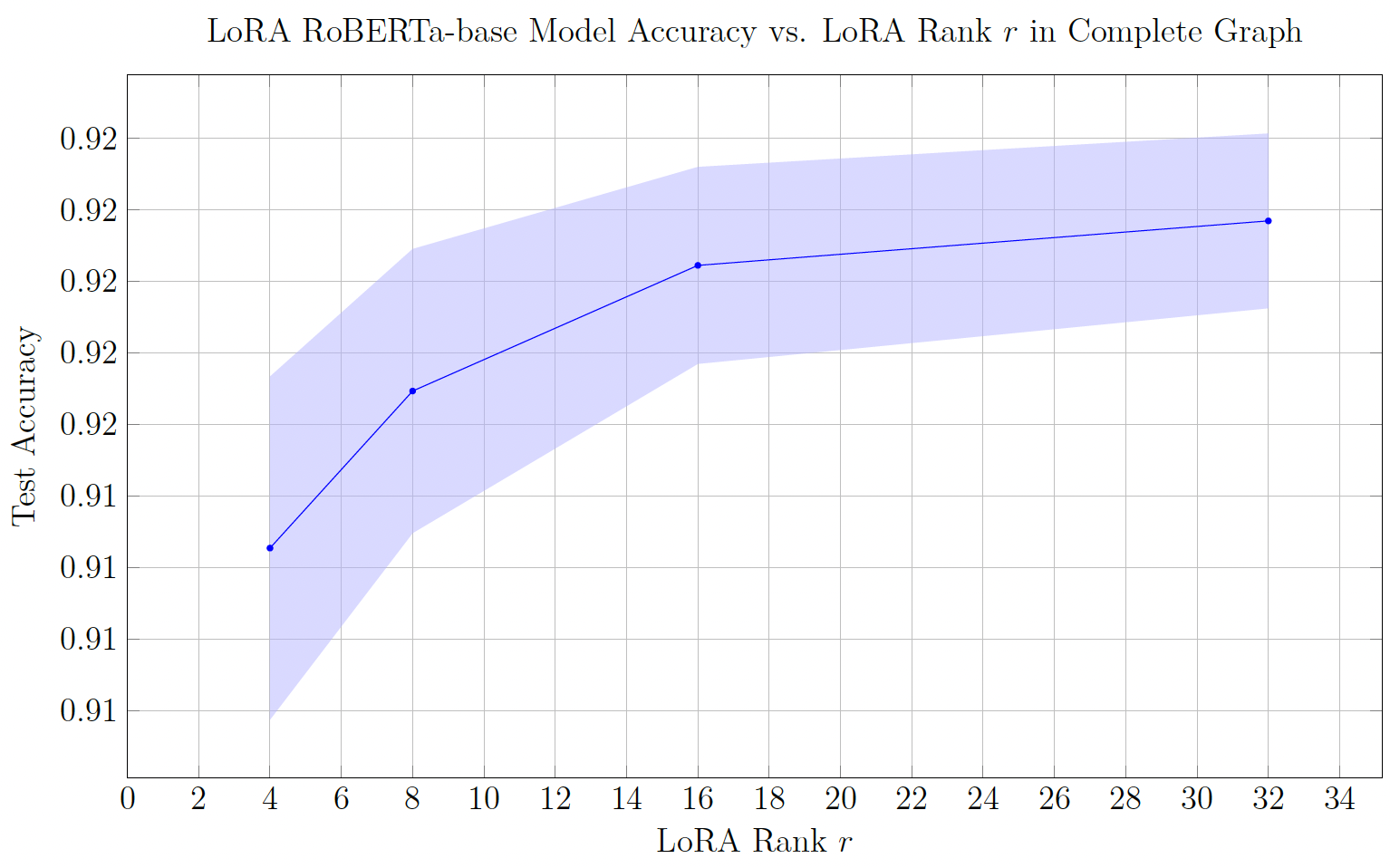}
\caption{Task accuracies under different LoRA ranks.}
\label{fig:rank}
\end{wrapfigure}
We conduct the ablation study for the LoRA rank by considering $r\in\set{4,8,16, 32}$. Fig.~\ref{fig:rank} presents the task accuracies achieved under different LoRA ranks in a complete graph. The results show that the performance remains largely unchanged across the these rank settings. This indicates that RW-LoRA is relatively insensitive to the choice of LoRA rank and can maintain competitive performance even with a small rank.

\section{Conclusion}\label{sec:Conclusion}
In this paper, we studied a decentralized random-walk-based LoRA framework. Unlike existing federated and gossip-based approaches that rely on multiple synchronized model copies, our method maintains a single mobile model that is updated sequentially across the network. We established convergence guarantees for non-convex objectives. Experimental results demonstrate that the proposed method achieves competitive performance while being substantially more resource-efficient than the gossip-based LoRA algorithm. Future work will explore additional information-theoretic aspects of random-walk LoRA algorithms. In particular, it would be interesting to study heterogeneous networks in which nodes use different LoRA ranks.

\bibliographystyle{IEEEtran}
\bibliography{references}

\onecolumn
\appendices

\section{Proof of Theorem~\ref{thm:Convergence}}\label{App:Convergence}
The proof consists of three parts. In the first part, we establish several useful lemmas. In the second part, we decompose the critical term $f(W^{(t+1)}) = f(B^{(t+1)}A^{(t+1)})$. In the third part, we apply a telescoping argument to obtain the desired result.

\subsection{Useful lemmas}
First of all, we provide the following useful lemmas.

\begin{lemma}\label{lem:Bounds0}
The expected squared norms of $\nabla_A \tilde{f}_v(W)$ and $\nabla_B \tilde{f}_v(W)$ are bounded:
\begin{align}
&\E\left[\|\nabla_A \tilde{f}_v(W)\|^2\right]\le b^2c^2,\label{eq:BoundsA}\\
&\E\left[\|\nabla_B \tilde{f}_v(W)\|^2\right]\le a^2c^2.\label{eq:BoundsB}
\end{align}
\end{lemma}
\begin{proof}
Since $W=BA$, the gradients of $\tilde{f}_v(W)$ with respect to $A$ and $B$ can be obtained by applying the chain rule. Specially, we have:
\begin{align}\label{eq:partialgradients}
&\nabla_A\tilde{f}_v(W)=B^T \nabla\tilde{f}_v(W),\nonumber\\
&\nabla_B\tilde{f}_v(W)=\nabla\tilde{f}_v(W) A^T.
\end{align}
Taking norms on both sides of \eqref{eq:partialgradients}, and using the submultiplicative property in Assumption~\ref{assu:Submultiplicative} together with Assumption~\ref{assu:BoundedPara}, we obtain
\begin{align}
\|\nabla_A\tilde{f}_v(W)\|\le& \|B\|\|\nabla\tilde{f}_v(W)\|\le b\|\nabla\tilde{f}_v(W)\|,\label{eq:nablaAtilde}\\
\|\nabla_B\tilde{f}_v(W)\|\le& \|A\|\|\nabla\tilde{f}_v(W)\|\le a\|\nabla\tilde{f}_v(W)\|.\label{eq:nablaBtilde}
\end{align}
According to Assumption~\ref{assu:Bounds0}, taking square and expectations on both sides of \eqref{eq:nablaAtilde} and \eqref{eq:nablaBtilde}, we obtain:
\begin{align*}
&\E\left[\|\nabla_A \tilde{f}_v(W)\|^2\right]\le b^2c^2,\\
&\E\left[\|\nabla_B \tilde{f}_v(W)\|^2\right]\le a^2c^2.
\end{align*}
\end{proof}

\begin{lemma}\label{lem:BoundDiffW}
The expected squared differences $\E[\|A^{(t+1)} - A^{(t)}\|^2]$, $\E[\|B^{(t+1)} - B^{(t)}\|^2]$, and $\E[\|W^{(t+1)} - W^{(t)}\|^2]$ are upper bounded as follows:
\begin{align}
&\E\left[\|A^{(t+1)}-A^{(t)}\|^2\right]\le\eta^2 K b^2c^2,\label{eq:DisA}\\
&\E\left[\|B^{(t+1)}-B^{(t)}\|^2\right]\le\eta^2 K a^2c^2,\label{eq:DisB}\\
&\E\left[\|W^{(t+1)}-W^{(t)}\|^2\right]\le 3K^2(\eta^2 b^4c^2 + \eta^2 a^4c^2 + K^2 \eta^4 a^2b^2c^4).\label{eq:DisW}
\end{align}
\end{lemma}
\begin{proof}
According to \eqref{eq:recursionAlocal} -- \eqref{eq:recursionB}, we have:
\begin{align*}
&\|A^{(t+1)} - A^{(t)}\|^2 =  \eta^2 \|\sum_{k=0}^{K-1}\nabla_A \tilde{f}_{v_t}(W^{(t,k)})\|^2,\\
&\|B^{(t+1)} - B^{(t)}\|^2 =  \eta^2 \|\sum_{k=0}^{K-1}\nabla_B \tilde{f}_{v_t}(W^{(t, k)})\|^2.
\end{align*}
Taking expectations on both sides of above equalities and Cauchy–Schwarz inequality, and based on \eqref{eq:BoundsA} and \eqref{eq:BoundsB}, we have \eqref{eq:DisA} and \eqref{eq:DisB}.

Next, we calculate $W^{(t+1)}-W^{(t)}$:
\begin{align*}
&W^{(t+1)}-W^{(t)}=B^{(t+1)}A^{(t+1)}-B^{(t)}A^{(t)}\\
&=\left(B^{(t)} - \eta \sum_{k=0}^{K-1}\nabla_B \tilde{f}_{v_t}(W^{(t,k)})\right)\left(A^{(t)} - \eta \sum_{k=0}^{K-1}\nabla_A \tilde{f}_{v_t}(W^{(t,k)})\right)-B^{(t)}A^{(t)}\\
&=- \eta \sum_{k=0}^{K-1}\nabla_A \tilde{f}_{v_t}(W^{(t)})B^{(t)} - \eta A^{(t)} \sum_{k=0}^{K-1} \nabla_B \tilde{f}_{v_t}(W^{(t)})\\
&+\eta^2\sum_{k_1=0}^{K-1}\sum_{k_2=0}^{K-1}\nabla_A \tilde{f}_{v_t}(W^{(t,k_1)})\nabla_B \tilde{f}_{v_t}(W^{(t,k_2)}).
\end{align*}
Taking norms on both sides, by the triangle inequality, we obtain:
\begin{align*}
\|W^{(t+1)}-W^{(t)}\|\le & \eta \|\sum_{k=0}^{K-1}\nabla_A \tilde{f}_{v_t}(W^{(t,k)})B^{(t)}\| + \eta \|A^{(t)} \nabla_B \sum_{k=0}^{K-1}\tilde{f}_{v_t}(W^{(t,k)})\|\\
+&\eta^2\|\sum_{k_1=0}^{K-1}\sum_{k_2=0}^{K-1}\nabla_A \tilde{f}_{v_t}(W^{(t,k_1)})\nabla_B \tilde{f}_{v_t}(W^{(t,k_2)})\|.
\end{align*}
Applying the submultiplicative property and the triangle inequality, we obtain
\begin{align*}
\|W^{(t+1)}-W^{(t)}\|\le & \eta b \sum_{k=0}^{K-1}\|\nabla_A \tilde{f}_{v_t}(W^{(t,k)})\| + \eta a \sum_{k=0}^{K-1}\|\nabla_B \tilde{f}_{v_t}(W^{(t,k)})\|\\
+&\eta^2\left(\sum_{k_1=0}^{K-1}\|\nabla_A \tilde{f}_{v_t}(W^{(t,k_1)})\|\right)\left(\sum_{k_2=0}^{K-1}\|\nabla_B \tilde{f}_{v_t}(W^{(t,k_2)})\|\right).
\end{align*}
For simplicity of notation, we denote
\begin{align*}
X_A =& \eta b \sum_{k=0}^{K-1}\|\nabla_A \tilde{f}_{v_t}(W^{(t,k)})\|,\quad X_B = \eta a \sum_{k=0}^{K-1}\|\nabla_B \tilde{f}_{v_t}(W^{(t,k)})\|\\
X_C=&\eta^2\left(\sum_{k_1=0}^{K-1}\|\nabla_A \tilde{f}_{v_t}(W^{(t,k_1)})\|\right)\left(\sum_{k_2=0}^{K-1}\|\nabla_B \tilde{f}_{v_t}(W^{(t,k_2)})\|\right).
\end{align*}
By Cauchy–Schwarz inequality, it follows that 
\begin{align*}
\|W^{(t+1)}-W^{(t)}\|^2\le 3X_A^2 + 3X_B^2 + 3X_C^2. 
\end{align*}
For $X_A^2$, again, by Cauchy–Schwarz inequality, we have:
\begin{align*}
X_A^2\le \eta^2 b^2\left(\sum_{k=0}^{K-1}1^2\right)\left(\sum_{k=0}^{K-1}\|\nabla_A \tilde{f}_{v_t}(W^{(t,k)})\|^2\right)=\eta^2 b^2 K \left(\sum_{k=0}^{K-1}\|\nabla_A \tilde{f}_{v_t}(W^{(t,k)})\|^2\right).
\end{align*}
Similar processes for $X_B^2$ and $X_C^2$, we have the following inequality:
\begin{align*}
\|W^{(t+1)}-W^{(t)}\|^2\le & 3\eta^2 b^2 K \left(\sum_{k=0}^{K-1}\|\nabla_A \tilde{f}_{v_t}(W^{(t,k)})\|^2\right) + 3\eta^2 a^2 K \left(\sum_{k=0}^{K-1} \|\nabla_B \tilde{f}_{v_t}(W^{(t,k)})\|^2\right)\\
+&3\eta^4 K^2 \left(\sum_{k=0}^{K-1}\|\nabla_A \tilde{f}_{v_t}(W^{(t,k)})\|^2\right)\left(\sum_{k=0}^{K-1}\|\nabla_B \tilde{f}_{v_t}(W^{(t,k)})\|^2\right).
\end{align*}
Taking expectations on both sides of above equalities, based on \eqref{eq:BoundsA} and \eqref{eq:BoundsB}, we obtain:
\begin{align*}
\E\left[\|W^{(t+1)}-W^{(t)}\|^2\right]\le 3K^2(\eta^2 b^4c^2 + \eta^2 a^4c^2 + K^2\eta^4 a^2b^2c^4).
\end{align*}
\end{proof}

\begin{lemma}[Lemma~4.10 in \cite{DecLoRA}]\label{lem:smoothAB}
Under Assumption~\ref{assu:Lsmooth}, each local fine-tuning objective $f_v$ is $Lb^2$-smooth with respect to $A$ when $B$ is fixed, and $La^2$-smooth with respect to $B$ when $A$ is fixed.
\end{lemma}

\begin{lemma}\label{lem:fsmooth}
Let $f(\cdot)$ be defined in \eqref{eq:LoRA1}. Then, $f$ is $L$-smooth. 
\end{lemma}
\begin{proof}
According to \eqref{eq:LoRA1},
\begin{align*}
\nabla f(W) - \nabla f(W') = \E_{v\sim\pi}\left[\nabla f_v(W) - \nabla f_v(W')\right].
\end{align*}
By Jensen's inequality applied to the convex norm function, we have:
\begin{align*}
\| \nabla f(W) - \nabla f(W')\| = \left\|\E_{v\sim\pi}\left[\nabla f_v(W) - \nabla f_v(W')\right]\right\|\le\E_{v\sim\pi}\left\|\nabla f_v(W) - \nabla f_v(W')\right\|.
\end{align*}
Based on Assumption~\ref{assu:Lsmooth}, it follows that
\begin{align*}
\| \nabla f(W) - \nabla f(W')\| \le \E_{v\sim\pi}\left[L\|W-W'\|\right] = L\|W-W'\|.
\end{align*}
\end{proof}
\begin{lemma}\label{lem:smoothfAB}
Let $f(\cdot)$ be defined in \eqref{eq:LoRA1}. Then, $f$ is $Lb^2$-smooth with respect to $A$ when $B$ is fixed, and $La^2$-smooth with respect to $B$ when $A$ is fixed.
\end{lemma}
\begin{proof}
The proof directly follows from Lemma~\ref{lem:smoothAB} and Lemma~\ref{lem:fsmooth}.
\end{proof}
Let $\pi_0$ denote the initial distribution, and let $\pi_t=P^t\pi_0$ denote the distribution at time step $t$. Let $\pi_{\min}=\min_{u\in\cV}\pi_u$.
\begin{lemma}\label{lem:randomgradient}
For $t\ge0$ and if $v_t\sim\pi_t$ for $d_{\text{TV}}(\pi_t, \pi)\le\frac{\pi_{\min}}{2}$, we have
\begin{align*}
\E\left[\|\nabla \tilde{f}_{v_t}(W)\|^2\right]\le 3\sigma^2 + 2\E\left[\|\nabla f(W)\|^2\right]
\end{align*}
\end{lemma}
\begin{proof}
Since $d_{\text{TV}}(\pi_t,\pi)\le\frac{\pi_{\min}}{2}$, then for any $v\in\cV$, we have
\begin{align*}
\Pr(v_t=v)\le \pi_v+\pi_v/2=\frac{3\pi_v}{2}.
\end{align*}
By the triangle inequality, Cauchy–Schwarz inequality, and Assumption~\ref{assu:BoundedVar}, it follows that
\begin{align*}
\E\left[\|\nabla \tilde{f}_{v_t}(W)\|^2\right]\le& 2\E\left[\|\nabla \tilde{f}_{v_t}(W)-\nabla f(W)\|^2\right]+2\E\left[\|\nabla f(W)\|^2\right]\\
\le&2\sum_{v\in\cV}\Pr(v_t=v)\sigma^2+2\E\left[\|\nabla f(W)\|^2\right]\\
=&3\sigma^2+2\E\left[\|\nabla f(W)\|^2\right].
\end{align*}
\end{proof}

\begin{lemma}\label{lem:Gradientf}
The expected squared norms of $\nabla f(W)$, $\nabla_A f(W)$ and $\nabla_B f(W)$ are bounded by:
\begin{align}
&\E\left[\|\nabla f(W)\|^2\right]\le2\sigma^2+2c^2,\label{eq:BoundsGW}\\
&\E\left[\|\nabla_A f(W)\|^2\right]\le 2b^2(\sigma^2+c^2),\label{eq:BoundsGA}\\
&\E\left[\|\nabla_B f(W)\|^2\right]\le 2a^2(\sigma^2+c^2).\label{eq:BoundsGB}
\end{align}
\end{lemma}
\begin{proof}
By the triangle inequality and Cauchy–Schwarz inequality, we have
\begin{align*}
\E\left[\|\nabla f(W)\|^2\right]\le &2\E\left[\|\nabla\tilde{f}_v(W)-\nabla f(W)\|^2\right]+ 2\E\left[\|\nabla\tilde{f}_v(W)\|^2\right]
\end{align*}
According to Assumption~\ref{assu:Bounds0} and Assumption~\ref{assu:BoundedVar}, we have
\begin{align*}
\E\left[\|\nabla f(W)\|^2\right]\le2\sigma^2+2c^2.
\end{align*}

Note that $W=BA$, we have:
\begin{align}\label{eq:GradientABW}
\nabla_A f(W)=B^T \nabla f(W),\,\,\nabla_B f(W)=\nabla f(W) A^T.
\end{align}
Applying the submultiplicative property and taking expectations on \eqref{eq:GradientABW}, we obtain:
\begin{align*}
&\E\left[\|\nabla_A f(W)\|^2\right]\le 2b^2(\sigma^2+c^2)\\
&\E\left[\|\nabla_B f(W)\|^2\right]\le 2a^2(\sigma^2+c^2).
\end{align*}

\subsection{Decomposition of $\E\left[f(B^{(t+1)}A^{(t+1)})\right]$}
Since $f$ is smooth, by the Descent Lemma\footnote{The descent lemma is a standard result for $L$-smooth differentiable functions. If $f$ is $L$-smooth, then for any $x,y$, 
\begin{align*}
f(y)\le f(x)+\langle \nabla f(x),y-x\rangle+\frac{L}{2}\|y-x\|^2.
\end{align*}}, we have: 
\begin{align}\label{eq:DLA}
\E\left[f(B^{(t+1)}A^{(t+1)})\right]\le& \E\left[f(B^{(t+1)}A^{(t)})\right] + \E\left\langle\nabla_A f(B^{(t+1)}A^{(t)}), A^{(t+1)}-A^{(t)} \right\rangle\nonumber \\
+&\frac{Lb^2}{2}\E\left[\|A^{(t+1)} - A^{(t)}\|^2\right].
\end{align}
In \eqref{eq:DLA}, the expectation $\E$ is taken over the target sampling $\pi$ and the randomness from the position of the RW at time $t$, i.e., $v_t$.
Again, by the Descent Lemma, we further have: 
\begin{align}\label{eq:DLB}
\E\left[f(B^{(t+1)}A^{(t)})\right]\le&\E\left[f(B^{(t)}A^{(t)})\right] + \E\left\langle\nabla_B f(B^{(t)}A^{(t)}), B^{(t+1)}-B^{(t)} \right\rangle\nonumber\\
+&\frac{La^2}{2}\E\left[\|B^{(t+1)} - B^{(t)}\|^2\right].
\end{align}
Combine \eqref{eq:DLA} and \eqref{eq:DLB}, we obtain:
\begin{align}\label{eq:DLAB}
\E\left[f(B^{(t+1)}A^{(t+1)})\right]\le&\E\left[f(B^{(t)}A^{(t)})\right] +\E\left\langle\nabla_A f(B^{(t+1)}A^{(t)}), A^{(t+1)}-A^{(t)} \right\rangle\nonumber\\
+&\E\left\langle\nabla_B f(B^{(t)}A^{(t)}), B^{(t+1)}-B^{(t)} \right\rangle +\frac{Lb^2}{2}\E\left[\|A^{(t+1)} - A^{(t)}\|^2\right]\nonumber\\
+&\frac{La^2}{2}\E\left[\|B^{(t+1)} - B^{(t)}\|^2\right].
\end{align}
By the Cauchy–Schwarz inequality followed by Young's inequality, we have
\begin{align*}
\E\left\langle\nabla_A f(B^{(t+1)}A^{(t)}), A^{(t+1)}-A^{(t)} \right\rangle\le& \frac{1}{2}\E\left[\|\nabla_A f(B^{(t+1)}A^{(t)})\|^2\right]\\
+&\frac{1}{2}\E\left[\|A^{(t+1)}-A^{(t)}\|^2\right].
\end{align*}
From \eqref{eq:DisA} in Lemma~\ref{lem:BoundDiffW} and \eqref{eq:BoundsGA} in Lemma~\ref{lem:Gradientf}, 
\begin{align}\label{eq:interA}
\E\left\langle\nabla_A f(B^{(t+1)}A^{(t)}), A^{(t+1)}-A^{(t)} \right\rangle \le b^2(\sigma^2+c^2)+\eta^2 K^2 \frac{b^2c^2}{2}.
\end{align}
According to the update rule \eqref{eq:recursionB}, we have 
\begin{align}\label{eq:interB1}
\E\left\langle\nabla_B f(B^{(t)}A^{(t)}), B^{(t+1)}-B^{(t)} \right\rangle &=\E\left\langle\nabla_B f(B^{(t)}A^{(t)}), -\eta \sum_{k=0}^{K-1}\nabla_B \tilde{f}_{v_t}(W^{(t,k)}) \right\rangle\nonumber\\
&=-\eta\E\left\langle\nabla_B f(B^{(t)}A^{(t)}), \sum_{k=0}^{K-1} \nabla_B \tilde{f}_{v_t}(W^{(t,k)}) \right\rangle\nonumber\\
&=\sum_{k=0}^{K-1}\E\left[-\eta\left\langle\nabla_B f(W^{(t)}), \nabla_B \tilde{f}_{v_t}(W^{(t,k)}) \right\rangle\right].
\end{align}
Let $\epsilon<\frac{1}{2}$, $\tau\ge\tau_{\mix}(\pi_{\min}\epsilon)$, and $t\ge\tau$, for any $k\in\set{0,1,\cdots,K-1}$, from the first equality in \cite[Appendix~C.1]{Even}, we have:
\begin{align*}
\E\left[-\eta\left\langle\nabla_B \tilde{f}_{v_t}(W^{(t,k)}), \nabla_B f(W^{(t)}) \right\rangle\right]&=\E\left[-\eta\left\langle\nabla_B \tilde{f}_{v_t}(W^{(t-\tau)}), \nabla_B f(W^{(t-\tau)}) \right\rangle\right]\\
&+\E\left[-\eta\left\langle\nabla_B \tilde{f}_{v_t}(W^{(t,k)}), \nabla_B f(W^{(t)}) - \nabla_B f(W^{(t-\tau)}) \right\rangle\right]\\
&+\E\left[-\eta\left\langle\nabla_B \tilde{f}(W^{(t,k)}) - \nabla_B \tilde{f}(W^{(t-\tau)}), \nabla_B f(W^{(t-\tau)}) \right\rangle\right]\\
&\triangleq E_1+E_2+E_3.
\end{align*}
By the the third inequality in \cite[Appendix~C.1]{Even},  we have:
\begin{align}\label{eq:term1mixingtime1}
E_1\le -\frac{\eta}{4}\E\left[\|\nabla_B f(W^{(t-\tau)})\|^2\right]+\eta\epsilon^2\sigma^2.
\end{align}
By utilizing Assumption~\ref{assu:Lsmooth} and the forth inequality in \cite[page~18]{Even}, we have:
\begin{align}\label{eq:term1mixingtime2}
E_2\le \frac{\eta^2L}{2}\tau\E\left[\|\nabla_B \tilde{f}_{v_t}(W^{(t,k)})\|^2\right]+\frac{\eta^2L}{2}\sum_{s=t-\tau}^{t-1}\sum_{k'=0}^{K-1}\E\left[\|\nabla_B \tilde{f}_{v_s}(W^{(s, k')})\|^2\right].
\end{align}
Similarly, by utilizing Assumption~\ref{assu:Lsmooth} and the sixth inequality in \cite[page~18]{Even}, we have:
\begin{align}\label{eq:term1mixingtime3}
E_3\le&\frac{\eta^2L}{2}\tau\E\left[\|\nabla_B f(W^{(t-\tau)})\|^2\right] +\frac{\eta^2L}{2}\sum_{s=t-\tau}^{t-2}\sum_{k'=0}^{K-1}\E\left[\|\nabla_B \tilde{f}_{v_s}(W^{(s, k)})\|^2\right]\nonumber\\
+& \frac{\eta^2L}{2}\sum_{k'=0}^{k-1}\E\left[\|\nabla_B \tilde{f}_{v_{t-1}}(W^{(t-1, k')})\|^2\right]\nonumber\\
\le&\frac{\eta^2L}{2}\tau\E\left[\|\nabla_B f(W^{(t-\tau)})\|^2\right]+\frac{\eta^2L}{2}\sum_{s=t-\tau}^{t-1}\sum_{k'=0}^{K-1}\E\left[\|\nabla_B \tilde{f}_{v_s}(W^{(s, k)})\|^2\right].
\end{align}
Based on Lemma~\ref{lem:BoundDiffW}, Lemma~\ref{lem:Gradientf}, and \eqref{eq:interA} -- \eqref{eq:term1mixingtime3}, we have:
\begin{align}\label{eq:term1mixingmain1}
\E\left[f(B^{(t+1)}A^{(t+1)})\right]\le& \E\left[f(B^{(t)}A^{(t)})\right]+K \eta \epsilon^2\sigma^2+ b^2(\sigma^2+c^2)\nonumber\\
+&\eta^2K\left(\frac{(2K+1)L}{2}\tau a^2c^2+\frac{c^2KL}{2}(a^4+b^4)+K^2\frac{b^2c^2}{2}\right)\nonumber\\
-&\left(\frac{\eta}{4}-\frac{\eta^2L}{2}\tau\right)K\E\left[\|\nabla_B f(W^{(t-\tau)})\|^2\right].
\end{align}

In \eqref{eq:DLA} and \eqref{eq:DLB}, we decompose $f(B^{(t+1)}A^{(t+1)})$ using the intermediate terms $f(B^{(t)}A^{(t+1)})$ and $f(B^{(t)}A^{(t)})$. By a similar analysis, we obtain:
\begin{align}\label{eq:term1mixingmain2}
\E\left[f(B^{(t+1)}A^{(t+1)})\right]\le& \E\left[f(B^{(t)}A^{(t)})\right]+K\eta\epsilon^2\sigma^2+a^2(\sigma^2+c^2)\nonumber\\
+&\eta^2K\left(\frac{(2K+1)L}{2}\tau b^2c^2+\frac{c^2KL}{2}(a^4+b^4)+K^2\frac{a^2c^2}{2}\right)\nonumber\\
-&\left(\frac{\eta}{4}-\frac{\eta^2L}{2}\tau\right)K\E\left[\|\nabla_A f(W^{(t-\tau)})\|^2\right].
\end{align}
Summing over \eqref{eq:term1mixingmain1} and \eqref{eq:term1mixingmain2}, we derive:
\begin{align}\label{eq:term1mixingmain3}
\E\left[f(B^{(t+1)}A^{(t+1)})\right]\le& \E\left[f(B^{(t)}A^{(t)})\right]+K\eta \epsilon^2\sigma^2+\frac{a^2+b^2}{2}(\sigma^2+c^2)\nonumber\\
+&\eta^2K\left(\frac{(2K+1)(a^2+b^2)L}{4}\tau c^2+\frac{c^2KL}{2}(a^4+b^4)+K^2\frac{c^2(a^2+b^2)}{4}\right)\nonumber\\
-&\left(\frac{\eta}{4}-\frac{\eta^2L}{2}\tau\right)K\left(\E\left[\|\nabla_A f(W^{(t-\tau)})\|^2\right]+\E\left[\|\nabla_B f(W^{(t-\tau)})\|^2\right]\right).
\end{align}

\subsection{Telescoping} 
Let $T\ge \tau^2$, $\eta\le\frac{1}{4L\sqrt{T}}$, we have
\begin{align*}
\frac{\eta}{8}\le\frac{\eta}{4}-\frac{\eta^2L}{2}\tau.
\end{align*}
Then, summing for $\tau\le t<T+\tau$, by telescoping \eqref{eq:term1mixingmain3}, we obtain:
\begin{align}\label{eq:term1mixingmain4}
&\frac{1}{T}\sum_{t=0}^{T-1}\left(\E\left[\|\nabla_A f(W^{(t)})\|^2\right]+\E\left[\|\nabla_B f(W^{(t)})\|^2\right]\right)\nonumber\\
&\le \epsilon^2\sigma^2 +\eta\Gamma+\frac{8\E\left[f(W^{(\tau)})\right]+4(a^2+b^2)(\sigma^2+c^2)/K}{\eta T},
\end{align}
where 
\begin{align*}
\Gamma = &2(2K+1)(a^2+b^2)L\tau c^2+4c^2KL(a^4+b^4)+2K^2c^2(a^2+b^2).
\end{align*}
Now, we upper bound $\E\left[f(W^{(\tau)})\right]$. According to Lemma~\ref{lem:fsmooth}, $f$ is $L$-smooth. 
By the Descent Lemma, we have:
\begin{align*}
f(W^{(\tau)}) - f(W^{(\tau-1)}) \le \langle \nabla f(W^{(\tau-1)}), W^{(\tau)}-W^{(\tau-1)}\rangle+\frac{L}{2}\|W^{(\tau)}-W^{(\tau-1)}\|^2.
\end{align*}
By Cauchy–Schwarz inequality, 
\begin{align*}
f(W^{(\tau)}) - f(W^{(\tau-1)}) \le  \left(\|\nabla f(W^{(\tau-1)})\| +\frac{L}{2}\|W^{(\tau)}-W^{(\tau-1)}\|\right)\|W^{(\tau)}-W^{(\tau-1)}\|,
\end{align*}
which implies
\begin{align*}
\E\left[f(W^{(\tau)})\right] - \E\left[f(W^{(\tau-1)})\right] \le  \left(\sqrt{\E\left[\|\nabla f(W^{(\tau-1)})\|^2\right]} +\frac{L}{2}\|W^{(\tau)}-W^{(\tau-1)}\|\right)\|W^{(\tau)}-W^{(\tau-1)}\|.
\end{align*}
Based on \eqref{eq:BoundsGW} in Lemma~\ref{lem:Gradientf} and \eqref{eq:DisW} in Lemma~\ref{lem:BoundDiffW}, we have:
\begin{align*}
\E\left[f(W^{(\tau)})\right] - \E\left[f(W^{(\tau-1)})\right] \le & 3\frac{L}{2}K^2(\eta^2 b^4c^2 + \eta^2 a^4c^2 + K^2 \eta^4 a^2b^2c^4)\\
+&3\sqrt{2\sigma^2+2c^2}K\sqrt{\eta^2 b^4c^2 + \eta^2 a^4c^2 + K^2 \eta^4 a^2b^2c^4}.
\end{align*}
By applying the inequality above repeatedly, we obtain:
\begin{align*}
\E\left[f(W^{(\tau)})\right]\le& \E\left[f(W^{(0)})\right]+3\tau\frac{L}{2}K^2(\eta^2 b^4c^2 + \eta^2 a^4c^2 + K^2 \eta^4 a^2b^2c^4)\\
+&3\tau\sqrt{2\sigma^2+2c^2}K\sqrt{\eta^2 b^4c^2 + \eta^2 a^4c^2 + K^2 \eta^4 a^2b^2c^4},
\end{align*}
which implies $\E\left[f(W^{(\tau)})\right]$ is bounded.

Finally, let $\epsilon<\frac{1}{2}$, $\tau\ge\tau_{\mix}(\pi_{\min}\epsilon)$, $T\ge \tau^2$, $\eta\le\frac{1}{4L\sqrt{T}}$, we obtain:
\begin{align*}
\frac{1}{T}\sum_{t=1}^{T}\left(\E\left[\|\nabla_{A} f(W^{(t)})\|^2\right]+\E\left[\|\nabla_{B} f(W^{(t)})\|^2\right] \right) \le \tilde{\cO}\left(\epsilon^2 + \frac{L\tau}{\sqrt{T}}+\frac{L\tau}{T}\right).
\end{align*}
\end{proof}

\end{document}